\documentclass{article}

   \usepackage[nonatbib, preprint]{neurips_2021}

\usepackage[utf8]{inputenc} 
\usepackage[T1]{fontenc}    
\usepackage{hyperref}       
\usepackage{url}            
\usepackage{booktabs}       
\usepackage{amsfonts,amsmath,amsthm}       
\usepackage{nicefrac}       
\usepackage{microtype}      
\usepackage{kbordermatrix}  
\usepackage{stmaryrd}       
\usepackage{subcaption}
\usepackage{wrapfig}
\usepackage{clipboard}      
\usepackage[dvipsnames]{xcolor}         

\usepackage{tikz}
    \usetikzlibrary{  cd,arrows,shapes,shapes.multipart, intersections,decorations,decorations.markings,decorations.pathreplacing,fit}
	\tikzset{thick/.style={line width=.6mm}}
	\tikzstyle{hugetensor}=[rounded rectangle,thick,draw=black,minimum width=40mm,minimum height = 3mm]
	\tikzstyle{squaretensor}=[rounded rectangle,thick,draw=black,minimum width=15mm,minimum height = 7mm]
	\tikzstyle{littletensor}=[circle,thick,draw=black,fill=red!30,inner sep=0pt,minimum size=10pt]
	\tikzstyle{medtensor}=[circle,thick,draw=black,fill=red!30,inner sep=0pt,minimum size=15pt]
	\tikzstyle{tinytensor}=[circle,thick,draw=black,fill=black,inner sep=0pt,minimum size = 4pt]

\hypersetup{
    hypertexnames=true,
    colorlinks=true,
    linkcolor=blue,  
    urlcolor=blue,
    citecolor=blue
}

\graphicspath{ {./figures/} }

\newcommand{\vecspace}[1]{\mathbb{F}^{#1}}
\newcommand{\genericspace}{\mathbb{F}^{d_1 \times d_2 \times \cdots \times d_n}}
\newcommand{\bigO}[1]{\mathcal{O}(#1)}
\newcommand{\normalization}{\mathcal{Z}}
\newcommand{\decoherence}{\mathcal{D}}
\newcommand{\decoherededges}{E_D}
\newcommand{\visibleedges}{E_V}
\newcommand{\hiddenedges}{E_H}
\newcommand{\cutedges}{E_C}
\newcommand{\TNnode}{A^{(v)}}
\newcommand{\TNBnode}{B^{(v)}}
\newcommand{\TNnodestar}{A^{(v)*}}
\newcommand{\N}{\mathbb{N}}
\newcommand{\R}{\mathbb{R}}
\newcommand{\C}{\mathbb{C}}
\newcommand{\F}{\mathbb{F}}
\newcommand{\clique}{\mathrm{Clq}}
\newcommand{\inc}{\mathrm{Inc}}
\newcommand{\ind}{x}
\newcommand{\edgevar}{\eta}
\newcommand{\condind}[3]{#1 \perp #2 \,|\, #3}

\newtheorem{theorem}{Theorem}
\newtheorem{definition}{Definition}
\newtheorem{corollary}{Corollary}

\theoremstyle{definition}

\title{Probabilistic Graphical Models and Tensor Networks: A Hybrid Framework}

\author{
  Jacob Miller$^*$\\
  Mila, Universit\'e de Montr\'eal \\
  Montr\'eal QC, Canada \\
  \texttt{jmjacobmiller@gmail.com} \\
  \And 
  Geoffrey Roeder$^*$\\
  Princeton University\\
  Princeton, NJ, USA \\
  \texttt{roeder@princeton.edu} \\
  \And
  Tai-Danae Bradley$^*$ \\
  X - The Moonshot Factory\\
  Mountain View, CA, USA \\
  \texttt{tai.danae@math3ma.com} \\
    \And
    \vspace{-0.4cm} \\
    {}$^\ast$Sandbox@Alphabet\\
    Mountain View, CA, USA \\
}

\begin{document}
\maketitle

\begin{abstract}

We investigate a correspondence between two formalisms for discrete probabilistic modeling: probabilistic graphical models (PGMs) and tensor networks (TNs), a powerful modeling framework for simulating complex quantum systems. The graphical calculus of PGMs and TNs exhibits many similarities, with discrete undirected graphical models (UGMs) being a special case of TNs. However, more general probabilistic TN models such as Born machines (BMs) employ complex-valued hidden states to produce novel forms of correlation among the probabilities. While representing a new modeling resource for capturing structure in discrete probability distributions, this behavior also renders the direct application of standard PGM tools impossible. We aim to bridge this gap by introducing a hybrid PGM-TN formalism that integrates quantum-like correlations into PGM models in a principled manner, using the physically-motivated concept of decoherence. We first prove that applying decoherence to the entirety of a BM model converts it into a discrete UGM, and conversely, that any subgraph of a discrete UGM can be represented as a decohered BM. This method allows a broad family of probabilistic TN models to be encoded as partially decohered BMs, a fact we leverage to combine the representational strengths of both model families. We experimentally verify the performance of such hybrid models in a sequential modeling task, and identify promising uses of our method within the context of existing applications of graphical models.

\end{abstract}


\section{Introduction}\label{sec:intro}

Probabilistic graphical models (PGMs) are a framework for encoding conditional independence information about multivariate distributions as graph-based representations, whose generality and interpretability have made them an indispensable tool for probabilistic modeling. Undirected graphical models (UGMs), also known as Markov random fields, form a general class of PGMs with a diverse range of applications in fields such as computer vision~\cite{wang2013markov}, natural language processing~\cite{wang2019bert}, and biology~\cite{mora2011biological}. More recently, the graphical structure of discrete UGMs has been shown to be closely related to that of tensor networks (TNs)~\cite{duality2018}, a state-of-the-art modeling framework first developed for quantum many-body physics~\cite{verstraete2008matrix,Orus2014}, whose use in machine learning---for example in model compression~\cite{novikov2015tensorizing,cichocki2016tensor}, proving separations in expressivity between deep and shallow learning methods~\cite{cohen2016expressive,levine2018deep}, and as standalone learning models~\cite{milesdavid,novikov2017exponential}---has been a subject of growing interest.

In this work, we explore the correspondence between UGMs and TNs in the setting of probabilistic modeling. Whereas UGMs are specifically designed to represent probability distributions, general TNs represent high-dimensional tensors whose values can be positive, negative, or even complex. While restricting TN parameters to take on non-negative values results in an exact equivalence with UGMs~\cite{duality2018}, it also limits their expressivity. More general probabilistic models built from TNs, as exemplified by the Born machine (BM)~\cite{bornmachine2018} model family, employ complex latent states that permit them to utilize novel forms of interference phenomena in structuring their learned distributions. While this provides a new resource for probabilistic modeling, it also limits the applicability of foundational PGM concepts such as conditional independence.

We make use of the physics-inspired concept of \emph{decoherence}~\cite{zurek2003decoherence} to develop a hybrid framework for probabilistic modeling, which allows for the coexistence of tools and concepts from UGMs alongside quantum-like interference behavior. We use this framework to define the family of \emph{decohered Born machines} (DBMs), which we prove is sufficiently expressive to reproduce any probability distribution expressible by discrete UGMs or BMs, along with more general families of TN-based models. We further show that DBMs satisfy a conditional independence property relative to its decohered regions, with the operation of decoherence permitting the values of latent random variables to be conditioned on in an identical manner as UGMs. Finally, we verify the empirical benefits of such models on a sequential modeling task.

\paragraph{Related Work}

Our work builds on the duality results of \cite{duality2018}, which establish a graphical correspondence between discrete UGMs and TNs, by further accounting for the distinct probabilistic behavior of both model classes. Much work across physics, machine learning, stochastic modeling, and automata theory has introduced and explored novel properties of quantum-inspired probabilistic models~\cite{zhao2010norm,bailly2011quadratic,FV2012,gao2017efficient,pestun2017tensor,pestun2017language,bornmachine2018,stoudenmire2018learning,stokes_terilla,benedetti2019generative,BST_2020,miller2021tensor,gao2021enhancing}, almost all of which explicitly or implicitly employ tensor networks. The relative expressivity of these models was explored in~\cite{glasser2019,adhikary2021quantum}, where quantum-inspired models were proven to be inequivalent to graphical models. Fully-quantum generalizations of various graphical models were investigated in~\cite{leifer2008quantum}.

\section{Preliminaries}

We work with real and complex finite-dimensional vector spaces $\vecspace{d}$, where $\F$ denotes one of $\R$ or $\C$ when the distinction is not needed. We take an $n$th order tensor, or \emph{$n$-tensor}, over $\F$ to be a scalar-valued map $T: [d_1] \times \cdots \times [d_n] \to \F$ from an $n$-fold Cartesian products of index sets, where $[d] := \{ 1, \ldots, d \}$ and where the vector space of all $n$-tensors is denoted by $\vecspace{d_1 \times \cdots \times d_n}$. Matrices, vectors, and scalars over $\F$ respectively correspond to $2$-tensors, $1$-tensors, and $0$-tensors, whereas \emph{higher-order tensors} refers to any $n$-tensor for $n > 2$. The \emph{elements} of $T$ are individual values of $T$ on input tuples, and written as $T_{x_1, \ldots, x_n} := T(x_1, \ldots, x_n) \in \F$, while the $i$th \emph{mode} of $T$ refers to the $i$th argument of $T$. The contraction of a vector $u \in \vecspace{d_i}$ with the $i$th mode of an $n$ tensor $T$ is the $(n-1)$-tensor $T'$ whose elements satisfy $T'_{x_1, \ldots, x_{i-1}, x_{i+1}, \ldots, x_n} = \sum_{x_i = 1}^{d_i} u_{x_i} T_{x_1, \ldots, x_{i}, \ldots, x_n}$. Although dense representations of $n$-tensors require $\bigO{d^n}$ parameters to specify, where $d = \max(d_1, \ldots, d_n)$, we will see later how tensor networks bypass this exponential scaling for many families of higher-order tensors. As one simple example, the \emph{tensor product} of any $n$-tensor $T \in \mathbb{F}^{d_1 \times \cdots \times d_n}$ and $m$-tensor $T' \in \vecspace{d'_1 \times \cdots \times d'_m}$ is the $(n+m)$-tensor $T \otimes T' \in \vecspace{d_1 \times \cdots \times d_n \times d'_1 \times \cdots \times d'_m}$ whose elements are given by $(T \otimes T')_{x_1, \ldots, x_n, x'_1, \ldots, x'_m} = T_{x_1, \ldots, x_n} T'_{x'_1, \ldots, x'_m}$. We use $\R_+$ to indicate the non-negative real numbers, and take the 2-norm of a tensor $T$ to be the scalar $\lVert T \rVert_2 = \sqrt{\sum_{x_1, \ldots, x_n} |T_{x_1, \ldots, x_n}|^2} \in \R_+$. Finally, we use $u^\dagger$ to indicate the conjugate transpose of a complex vector or matrix $u$.

We focus exclusively on undirected graphs $G = (V, E)$, whose vertex and edge sets are denoted by $V$ and $E$. In anticipating the needs of tensor networks, we allow the existence of edges which are incident to only one node, which we refer to as \emph{visible} (i.e. dangling) \emph{edges}. We use $\visibleedges \subseteq E$ to indicate the set of all visible edges, and $\hiddenedges := E - \visibleedges$ to indicate the set of all \emph{hidden edges}, which are edges adjacent to two nodes. Graphs without dangling edges will be called \emph{proper} graphs. For any node $v \in V$, we denote the set of edges incident to $v$ by $\inc(v)$. A \emph{clique} of $G$ is a maximal subset $C \subseteq V$ such that every pair of nodes $v, v' \in C$ are connected by an edge, and we use $\clique(G)$ to denote the set of all cliques of $G$. We define a \emph{cut set} of $G$ to be any set of edges $\cutedges \subseteq E$ such that the removal of all edges in $\cutedges$ from $G$ partitions the graph into two disjoint non-empty sub-graphs.

We work with random variables (RVs), indicated by uppercase letters such as $X, Y, Z$, and their possible outcomes, indicated by lowercase equivalents such as $x, y, z$. RVs and their outcomes are frequently indexed with values chosen from an index set, for example $i \in \mathcal{I} = \{1, \ldots, n\}$, in which case the notation $X_{\mathcal{I}}$ indicates the joint RV $(X_1, \ldots, X_n)$. A similar notation is used for multivariate functions $f(x_{\mathcal{I}}) := f(x_1, \ldots, x_n)$, as well as for tensor elements $T_{x_{\mathcal{I}}} := T_{x_1, \ldots, x_n}$, and a related notation $\vecspace{\times_{i \in \mathcal{I}} d_i} := \vecspace{d_1 \times \cdots \times d_n}$ is used for spaces of tensors. Given three disjoint sets of random variables $X_A, X_B, X_C$, we use $\condind{X_A}{X_B}{X_C}$ to indicate the conditional independence of $X_A$ and $X_B$ given $X_C$, and $X_A \perp X_B$ to indicate the (unconditional) independence of $X_A$ and $X_B$.

\subsection{Undirected Graphical Models}

Probabilistic graphical models (PGMs) represent multivariate probability distributions using a proper graph $G = (V, E)$ whose nodes $v$ each correspond to distinct RVs $X_v$. We focus on undirected graphical models (UGMs), whose probability distributions are determined by a collection of \emph{clique potentials} $\phi_C: X_C \to \R_+$, non-negative valued functions from the RVs associated with nodes in $C$, where $C$ ranges over all cliques of $G$. Given a UGM with clique potentials $\phi_C$ defined on a graph $G$ with $n$ nodes, the probability distribution represented by the UGM is
\begin{equation}
\label{eq:ugm_prob}
    P(\ind_1,\ldots,\ind_n) = \frac{1}{\normalization} \prod_{C \in \clique(G)} \phi_C(\ind_C),\ \ \textrm{where}\ \normalization =\!\! \sum_{\ind_1,\ldots, \ind_n} \prod_{C \in \clique(G)} \phi_C(\ind_C).
\end{equation}
For brevity, we will often omit normalization factors such as $\normalization$ in the following, with the understanding that such terms must ultimately be added to ensure a valid probability distribution. UGMs satisfy an intuitive conditional independence property involving disjoint subsets of nodes $A, B, C \subseteq V$ for which the removal of $C$ leaves the nodes of $A$ and $B$ in separate disconnected subgraphs of $G$. In this case, the RVs associated with these nodes satisfy $\condind{X_A}{X_B}{X_C}$.

While the definition above is the standard presentation of UGMs, to permit an easier comparison with tensor networks we will more frequently view them using a dual graphical formulation. In this dual picture, nodes represent clique potentials $\phi_C$ and edges represent RVs $X_i$.

\section{Tensor Networks}

Tensor networks (TNs) provide a general means to efficiently represent higher-order tensors in terms of smaller tensor cores, in much the same way as UGMs efficiently represent multivariate probability distributions in terms of smaller clique potentials. \emph{Tensor contraction} is crucial in the structure of TNs, and generally involves the multiplication of an $n$-tensor $T \in \vecspace{d_1 \times \cdots \times d_n}$ and an $m$-tensor $T' \in \vecspace{d'_1 \times \cdots \times d'_m}$ along modes $k, k'$ of equal dimension $d_k = d'_{k'}$, to yield a single output $(n+m-2)$-tensor $T''$ whose elements are given by
\begin{equation}
\label{eq:contraction}
    T''_{\ind_1 \cdots \ind_{k-1} \ind_{k+1} \cdots \ind_n \ind'_1 \cdots \ind'_{k'-1} \ind'_{k'+1} \cdots \ind'_{m}} = \sum_{\ind_k = 1}^{d_k} T_{\ind_1 \cdots \ind_{k-1} \ind_k \ind_{k+1} \cdots \ind_n} T'_{\ind'_1 \cdots \ind'_{k'-1} \ind_k \ind'_{k'+1} \cdots \ind'_m}.
\end{equation}
\noindent Although appearing complex in its general form, it is worth verifying from Equation~(\ref{eq:contraction}) that tensor contraction generalizes matrix-vector and matrix-matrix multiplication, along with vector inner product and scalar multiplication. Tensor contraction is associative, in the sense that multiple contractions between multiple tensors yield the same output regardless of the order of contraction, with different contraction orderings often having vastly different memory and compute requirements.

\emph{Tensor network diagrams}~\cite{penrose71} provide an intuitive formalism for reasoning about computations involving tensor contraction using undirected graphs. In a TN diagram, each $n$-tensor $T \in \genericspace$ is represented as a node of degree $n$, and each mode of $T$ is represented as an edge incident to $T$. Tensor contraction between two tensors along a pair of modes is depicted by connecting the corresponding edges of the nodes, with the actual operation of tensor contraction depicted by merging the nodes representing both input tensors into a single node which shares the visible edges of both input nodes. In this manner, a TN diagram with $n$ visible edges and any number of hidden edges specifies a sequence of tensor contractions whose output will always be an $n$-tensor. For example, the TN diagram $\begin{tikzpicture}[y=0.15cm,x=0.3cm,baseline={([yshift=-0.5ex] current bounding box.center)}]
    \draw[thick] (-1,0) -- (1,0);
    \node[littletensor, fill=green!20, minimum size = 6pt] at (0,0) {};
    \node[] at (2,0) {=};
    \draw[thick] (3,0) -- (7.4,0);
    \node[littletensor, fill=blue!20, minimum size = 6pt] at (4,0) {};
    \node[littletensor, fill=red!20, minimum size = 6pt] at (5.2,0) {};
    \node[littletensor, fill=yellow!20, minimum size = 6pt] at (6.4,0) {};
    \end{tikzpicture}$ 
expresses a tensor contraction used in the SVD to express a matrix as the product of three smaller matrices. As a degenerate case, the tensor product of two tensors is depicted by drawing them adjacent to each other, with no connected edges. 

\begin{figure}
\centering
\includegraphics[width=\textwidth]{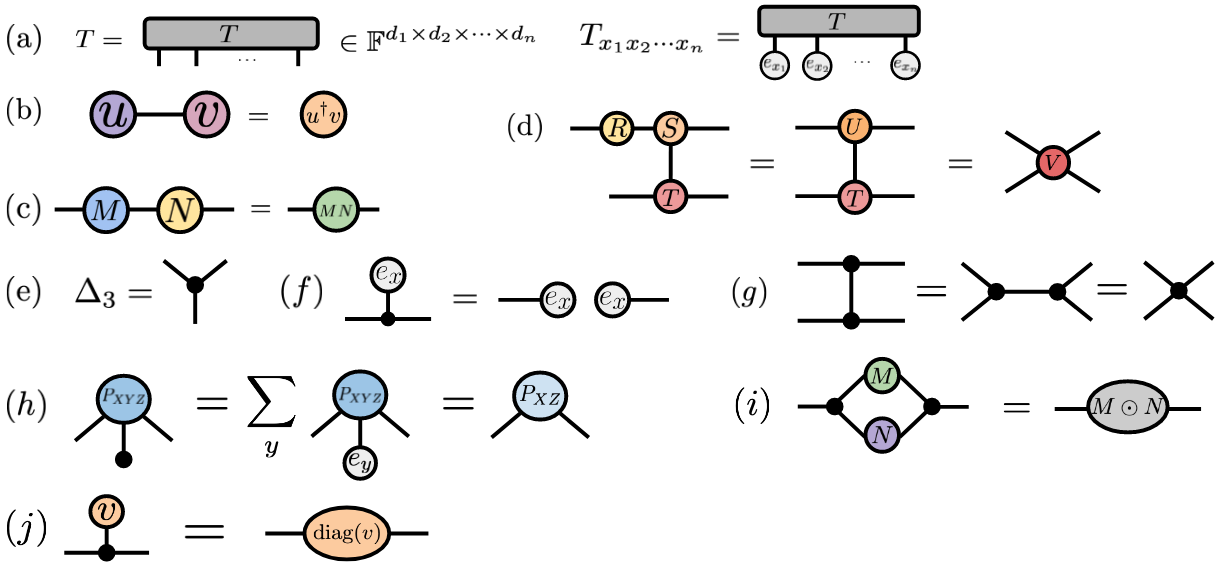}
\caption{Tensor network and copy tensor notations. (a) Basic notations for TNs, where nodes represent tensors and edges represent tensor modes. Vectors $e_x$ denote elements of an orthonormal basis used to express tensors as arrays. (b--c) Vector inner products $u^\dagger v$ and matrix multiplication $MN$ are simple examples of tensor contraction. (d) Tensor contraction is associative, with the contraction of tensors $R$, $S$, and $T$ by first contracting $R$ and $S$ (shown) giving the same result as first contracting $S$ and $T$. (e) Copy tensors $\Delta_n = \sum_x (e_x)^{\otimes n}$ are denoted by a black dot with $n$ edges. (f) Copy tensors act on basis vectors $e_x$ by copying them to all visible edges and (g) they permit any connected network of copy tensors to be arbitrarily rearranged, provided the total number of visible edges remain unchanged. Copy tensors also allow the graphical expression of basis-dependent operations, including (h) marginalizing over a RV in a probability distribution, (i) the element-wise product of tensors, and (j) the creation of diagonal matrices from a vector of diagonal values.}
\label{fig:TN_diagrams}
\end{figure}

\emph{Tensor networks} use a fixed TN diagram to efficiently parameterize a family of higher-order tensors in terms of a family of smaller dense tensor \emph{cores}, as stated in the following:

\begin{definition}
\label{def:tn}
A tensor network consists of a graph $G = (V, E)$, along with a positive integer valued map $d_{(-)}: E \to \N_{>0}$ assigning each edge $\edgevar$ to a vector space of dimension $d_\edgevar$, and a map $A^{(-)}: v \mapsto \vecspace{\times_{\edgevar \in \inc(v)} d_\edgevar}$ assigning each node $v$ to a tensor core $A^{(v)}$ whose shape is determined by the dimensions assigned to edges incident to $v$. The tensor represented by a tensor network with $m$ nodes is that resulting from a contraction of all $m$ tensor cores according to the tensor network diagram defined by $G$.
\end{definition}

Dimensions $d_i$ assigned to hidden edges are referred to as \emph{bond dimensions}, and for a fixed graph $G$ they represent the primary hyperparameters setting the tradeoff between a TN's compute/memory efficiency and its expressivity. A simple but illustrative example of a TN is a low-rank matrix factorization, whose graph $G$ is the line graph on two nodes $\begin{tikzpicture}[y=0.15cm,x=0.3cm,baseline={([yshift=-0.5ex] current bounding box.center)}]
    \draw[thick] (-1,0) -- (2.2,0);
    \node[littletensor, fill=blue!20, minimum size = 6pt] at (0,0) {};
    \node[littletensor, fill=yellow!20, minimum size = 6pt] at (1.2,0) {};
    \end{tikzpicture}$, and whose single hidden edge is associated with a bond dimension $r$ equal to the rank of the parameterized matrices.

\paragraph{Copy Tensors} A simple family of tensors plays a key role in understanding the relationship between UGMs and TNs. Given an orthonormal basis $\mathcal{B} = \{e_1,\ldots,e_d\}$ for a vector space $\vecspace{d}$, for each $n \geq 1$ we define the $n$th order \emph{copy tensor} associated with $\mathcal{B}$ to be $\Delta_n := \sum_{\ind=1}^d (e_\ind)^{\otimes n}$. When $\Delta_n$ is contracted with any of the $d$ basis vectors $e_\ind$, the result is a tensor product $(e_\ind)^{\otimes n-1}$ of $n-1$ independent copies of $e_\ind$ (Figure~\ref{fig:TN_diagrams}f). This convenient property only holds for vectors chosen from the basis defining the copy tensor, leading to a one-to-one correspondence between copy tensor families and orthonormal bases~\cite{coecke_pavlovic_vicary_2013}. The copy tensors $\Delta_1$ and $\Delta_2$ respectively correspond to the $d$-dimensional all-ones vector and identity matrix, with the former allowing the expression of sums over tensor elements. 

A general copy tensor $\Delta_n$ is depicted graphically as a single black dot with $n$ edges (Figure~\ref{fig:TN_diagrams}e), while $\Delta_2$ is depicted as an undecorated edge. These satisfy a useful closure property under tensor contraction, whereby any connected network of copy tensors is identical to a single copy tensor with the same number of visible edges~\cite[Theorem 6.45]{fong2019invitation}. This property allows connected networks of copy tensors to be rearranged in any manner, so long as the number of visible edges remains unchanged (Figure~\ref{fig:TN_diagrams}g). While general tensor network diagrams remain unaffected by changes of basis in the hidden edges, the use of copy tensors permits the graphical depiction of a larger family of basis-dependent linear algebraic operations (Figure~\ref{fig:TN_diagrams}h--j).

\subsection{Undirected Graphical Models as Non-Negative Tensor Networks}
\label{sec:ugm_as_nntn}

\begin{figure}
\centering
\includegraphics[width=\textwidth]{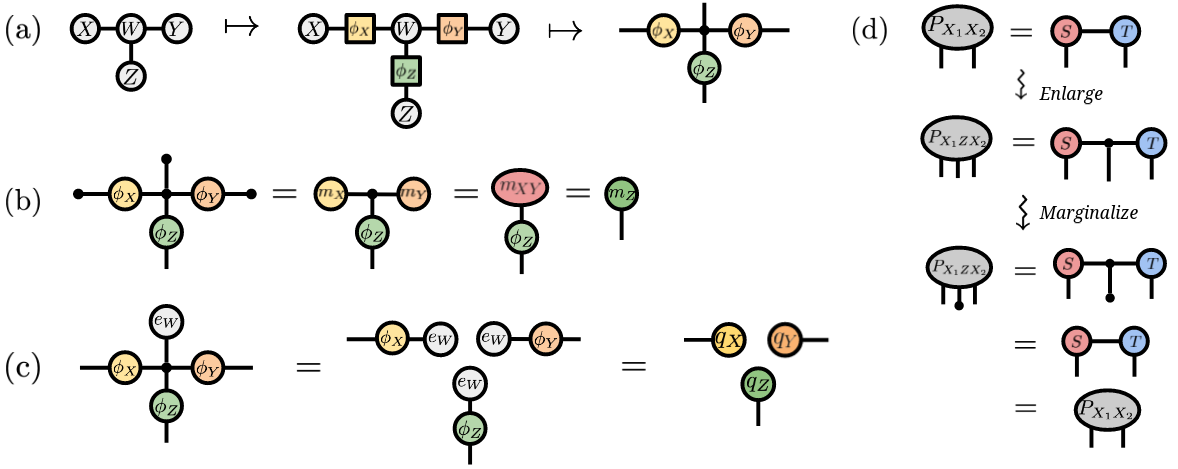}
\caption{Tensor network description of UGMs. (a) Example of the duality between UGM-style graphical notation, where nodes are associated with RVs, and TN-style graphical notation, where nodes are associated with clique potentials. Factor graphs act as an intermediate representation, with duality simply interchanging variable and factor nodes. (b) Marginalization of a probabilistic model is represented in TN notation by contracting the corresponding visible edges by first-order copy tensors. (c) Conditioning is represented by contracting the corresponding visible edges by outcome-dependent basis vectors, with conditional independence of the resulting distribution achieved through properties of copy tensors. (d) Any TN with non-negative cores can be converted to a UGM by promoting its hidden edges into visible edges by the use of third-order copy tensors. Marginalizing these new latent RVs recovers the original distribution over the visible edges.}
\label{fig:copy}
\end{figure}

Discrete multivariate probability distributions are an important example of higher-order tensors, with the individual probabilities $P(x_1, x_2, \ldots, x_n)$ of a distribution over $n$ discrete RVs forming the elements of an $n$-tensor. More generally any \emph{non-negative tensor} $T$, whose elements all satisfy $T_{\ind_1\ind_2\cdots \ind_n} \geq 0$, can be converted into a probability distribution $P_T$ by normalizing as $P_T(\ind_1,\ind_2,\ldots,\ind_n) = \frac{1}{\normalization} T_{\ind_1\ind_2\cdots \ind_n}$, where $\normalization =\!\! \sum_{\ind_1,\ind_2,\ldots, \ind_n}\!\! T_{\ind_1\ind_2\cdots \ind_n}$.
%

The connection between multivariate probability distributions and the structure of higher-order tensors extends further, with the independence relation $X_A \perp X_B$ between two disjoint sets of RVs being equivalent to the factorization of their joint probability distribution as the tensor product $P(x_A, x_B) = P(x_A) \otimes P(x_B)$. Methods used to efficiently represent and learn higher-order tensors, such as tensor networks, can be applied to probabilistic modeling, provided that some means of parameterizing only non-negative tensors is employed. We discuss two important approaches for achieving this probabilistic parameterization, one equivalent to undirected graphical models and the other to Born machines.

It was shown in \cite[Theorem 2.1]{duality2018} that the data defining a UGM is equivalent to that defining a TN, but with dual graphical notations that interchange the roles of nodes and edges. Converting from a UGM to an equivalent TN involves expressing each clique potential $\phi_C$ on a clique $C$ of size $k$ as a $k$th-order tensor core $A^{(v_C)}$, depicted as a degree-$k$ node $v_C$ of the TN diagram. Meanwhile, each UGM node representing a discrete RV is replaced by a copy tensor\footnote{Copy tensors were used implicitly in \cite{duality2018}, in the form of hyperedges within a defining hypergraph.} of degree equal to the number of clique potentials the RV occurs in, plus one additional visible edge permitting the values of the RV to appear in the probability distribution described by the TN (Figure~\ref{fig:copy}a). Since every tensor core consists of non-negative elements, the resultant TN is guaranteed to describe a non-negative higher-order tensor. We refer to this family of TN models as \emph{non-negative tensor networks}.

In the dual graphical notation of TNs, marginalization of and conditioning on RVs in UGMs is achieved by contracting each visible edge of the associated TN with either a first-order copy tensor $\Delta_1 = \sum_{x} e_x$ (marginalization) or an outcome-dependent basis vector $e_x$ (conditioning) respectively. Computing the resulting distribution over the remaining RVs is then a straightforward application of tensor contraction~\cite{duality2018}, where any nodes of the TN with no remaining visible edges are merged together (Figure~\ref{fig:copy}b). Furthermore, since variables are associated to copy tensors, the conditional independence property of UGMs can be proven using the copying property of copy tensors (Figure~\ref{fig:copy}c). The appropriate formulation of conditional independence is slightly different in the dual TN notation, owing to the association of RVs to edges rather than nodes. In this graphical framework, conditional independence arises when a conditioning set of RVs $X_C$ form a cut set of the underlying TN graph, in which case the RVs $X_A$ and $X_B$ associated with the two partitions of the graph induced by this cut set will satisfy $\condind{X_A}{X_B}{X_C}$.

The reverse direction of converting non-negative TNs into UGMs is also straightforward, though care is needed with the treatment of hidden edges that aren't connected to copy tensors. In such cases, we can replace any hidden edge with a third-order copy tensor $\Delta_3$, yielding a new visible edge which encodes a latent RV in an enlarged distribution
. This enlargement process is reversible, in the sense that marginalizing over all latent RVs associated with hidden edges yields the original distribution, allowing hidden edges of a non-negative TN to be treated as visible edges without any loss of generality (Figure~\ref{fig:copy}d). We will see shortly that this property is not shared by more general probabilistic TN models.

\section{Born Machines}

\begin{figure}
\centering
\includegraphics[width=\textwidth]{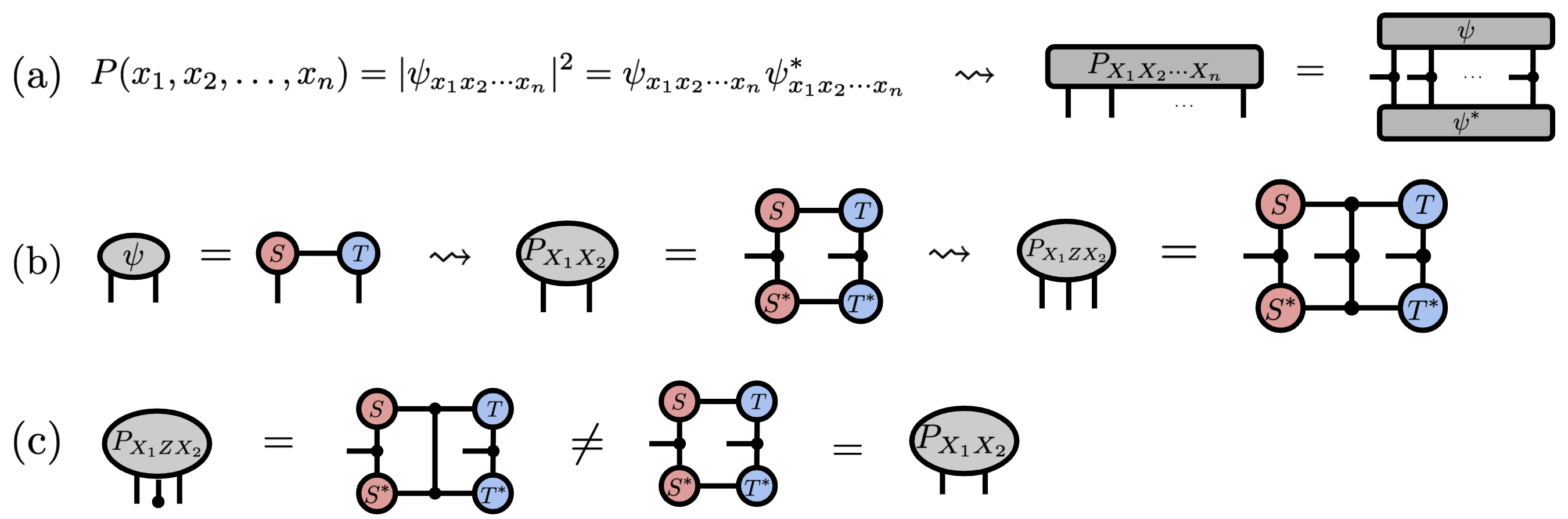}
\caption{Overview of Born machines, a family of TN-based probabilistic models. (a) Born machines represent a general higher-order tensor $\psi$ as a TN, whose elements are converted to probabilities via the Born rule. This can be used to express the probability distribution itself as a composite TN diagram. (b--c) Unlike non-negative TNs, any attempt to read out the hidden edges of BMs as latent RVs alters the overall distribution, a manifestation of the ``observer effect'' of quantum mechanics. Converting a hidden edge to a RV and then marginalizing results in a different distribution.} 
\label{fig:fig3}
\end{figure}

While UGMs represent one means of parameterizing non-negative tensors for probabilistic modeling, an alternate approach is suggested by quantum physics. Discrete quantum systems are fully described by complex-valued \emph{wavefunctions}, higher-order tensors which yield probabilities under the Born rule of quantum mechanics. The efficacy of TNs in learning quantum wavefunctions inspired the Born machine (BM) model~\cite{bornmachine2018}.

\begin{definition}
\label{def:bm}
A Born machine consists of a tensor network over a graph $G$ containing $n$ visible edges, whose associated tensor $\psi \in \vecspace{d_1 \times \cdots \times d_n}$ is converted into a probability distribution via the Born rule $P(x_1, \ldots, x_n) = \frac{1}{\lVert \psi \rVert_2^2} |\psi_{x_1 \cdots x_n}|^2$, where $\lVert \psi \rVert_2$ denotes the 2-norm of $\psi$.
\end{definition}

The Born rule permits the (unnormalized) probability distribution associated with a BM to be expressed as a single \emph{composite TN}, consisting of two copies of the TN parameterizing $\psi$, one with all core tensor values complex-conjugated, and where all pairs of visible edges have been merged via third-order copy tensors (Figure~\ref{fig:fig3}a). Expressing the BM distribution as a single composite TN allows efficient marginal and conditional inference procedures to be applied in a manner analogous to UGMs, namely by contracting the visible edges of the composite TN with vectors $\Delta_1$ (marginalization) or $e_{x_i}$ (conditioning), and then contracting regions of the TN with no remaining visible edges. The ``doubled up'' nature of the composite TN means that intermediate states occurring during inference are described by \emph{density matrices}, which are positive semidefinite matrices employed in quantum mechanics whose non-negative diagonal entries correspond to (unnormalized) probabilities, and whose off-diagonal elements are referred to as ``coherences.''

The existence of non-zero coherences gives BMs the ability to utilize quantum-like interference phenomena in modeling probability distributions, but also makes it difficult to interpret the operation of a BM by assigning latent RVs to its edges, as was possible with non-negative TNs. While we can force a new RV into existence by extracting the diagonal elements of intermediate density matrices using copy tensors (Figure~\ref{fig:fig3}b), this causes the elimination of all coherences in density matrices passing through the edge, with the result that the distribution after marginalizing the new latent variable differs from the original BM distribution (Figure~\ref{fig:fig3}c). This fact, which can be seen as an analogue of the measurement-induced ``observer effect'' in quantum mechanics, represents a tradeoff between expressivity and interpretability in BMs that isn't available in PGMs.

\section{A Hybrid Framework}

\begin{figure}
\centering
\includegraphics[width=\textwidth]{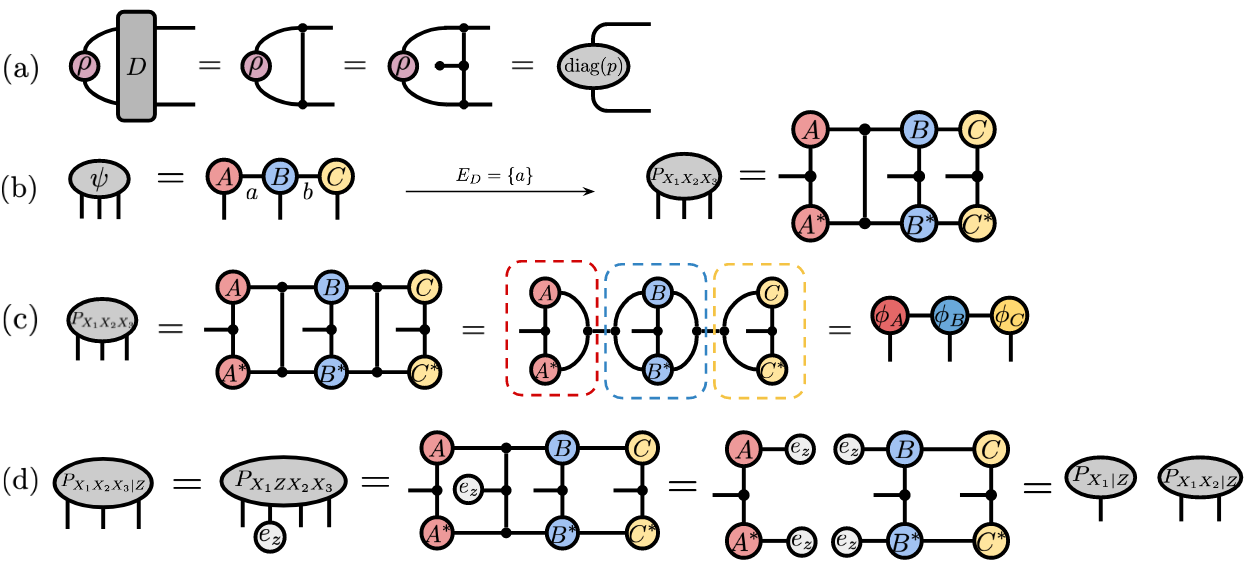}
\caption{Decoherence and the decohered Born machine (DBM) model. (a) The decoherence operator $\decoherence$, which removes coherences from hidden states in BMs, leaving a diagonal matrix which encodes a latent RV. Decoherence operators permits the readout of latent RVs in a reversible manner, with marginalization of the latent RV yielding the original distribution (third diagram). (b) Examples of decohered Born machines (DBMs) based on a three-core TN with hidden edges $a$ and $b$. Choosing the decohered edge set $\decoherededges = \{ a \}$ results in a decoherence operator being placed in the location corresponding to edge $a$ in the composite TN, which allows the decohered edge to be expressed as a new latent RV. (c) Sketch of the proof of Theorem~\ref{thm:fdbm_as_ugm}, that every fully-decohered BM is equivalent to a UGM. Copy tensor rewriting rules permit the factorization of fully-decohered BMs into non-negative valued tensors which form clique potentials of an equivalent UGM. (d) Example of the conditional independence property for the DBM above with $\decoherededges = \{a\}$. Conditioning on $Z = z$ for the latent RV at decohered edge $a$ leads to the conditioning value being copied to all attached cores, which in turn leads to a factorization of the conditional distribution into two independent pieces.}
\label{fig:proofs}
\end{figure}

While the graphical structure of Born machines is useful for defining \emph{area laws}, which characterize the attainable mutual information between subsets of RVs~\cite{eisert2010colloquium,lu2021}, they do not permit the formulation of conditional independence results, something which is a major benefit of PGMs. A primary reason for this is the inability to freely assign latent RVs to the hidden edges of BMs without disturbing the original distribution. However, by accounting for this disturbance in a principled manner, it is possible to combine the representational advantages available to BMs with the conditional independence guarantees available to PGMs.

A crucial tool is the concept of decoherence, whereby all off-diagonal coherences of a hidden density matrix are set to zero, leaving only a probability distribution on the diagonals of the operator. This can be carried out by the action of a \emph{decoherence operator}, a fourth-order copy tensor acting on density matrices which we write as $\decoherence$ (Figure~\ref{fig:proofs}a). The operator $\decoherence$ is the natural result of converting a hidden edge of a BM into a latent RV and then marginalizing.  We can use this idea to decohere certain edges of a BM in advance, leading to the notion of \emph{decohered Born machine} models (Figure~\ref{fig:proofs}b).

\begin{definition}
\label{def:dbm}
A \emph{decohered Born machine} (DBM) consists of a Born machine over a graph $G$ along with a subset of hidden edges $\decoherededges \subseteq \hiddenedges$, referred to as the model's \emph{decohered edges}. The probability distribution represented by a DBM is given by the composite TN associated to the original BM, but with each pair of hidden edges in $\decoherededges$ replaced by a decoherence tensor $\decoherence$. A DBM for which $\decoherededges = \hiddenedges$ is referred to as a \emph{fully-decohered} Born machine.
\end{definition}

Having Definition~\ref{def:dbm} in hand, we would like to first characterize the expressivity of DBMs. It is clear that standard BMs are a special case of DBMs, where the decohered edge set is taken to be empty. On the other hand, we show in the following two results that fully-decohered BMs are equivalent in expressive power to discrete UGMs.

\begin{theorem}
\label{thm:fdbm_as_ugm}
\Copy{fdbm_as_ugm}{The probability distribution expressed by a fully-decohered Born machine with tensor cores $\TNnode$, one for each node $v \in V$, is identical to that of a discrete undirected graphical model with clique potentials of the same shape, and whose values are given by $\phi_{C}(x_{C}) = | \TNnode_{x_{C}} |^2$, where $x_C$ contains the RVs from all edges adjacent to $v$.}
\end{theorem}

The proof of Theorem~\ref{thm:fdbm_as_ugm} is given in the supplemental material, with the basic idea illustrated in Figure~\ref{fig:proofs}c. Each decoherence operator $\decoherence$ can be written as the product of two third-order copy tensors, each of which can be assigned to one pair of TN cores adjacent to the decohered edge. In the case that all edges of a DBM are decohered, these copy tensors allow each pair of cores $\TNnode$ and $\TNnodestar$ to be replaced by their element-wise product, giving an effective clique potential with non-negative values. The UGM formed by these clique potentials has the same graphical structure as the TN describing the original BM (up to graphical duality). Conversely, the correspondence given in Theorem~\ref{thm:fdbm_as_ugm} suggests a simple method for representing any discrete UGM as a fully-decohered BM.

\begin{corollary}
\label{thm:ugm_as_dbm}
\Copy{ugm_as_dbm}{The probability distribution of any discrete undirected graphical model with clique potentials $\phi_{C}(x_C)$ is identical to that of any fully-decohered Born machine with tensor cores of the same shape, and whose elements are given by $\TNnode_{x_C} = \exp\left( 2\pi i\, \theta_C(x_C) \right) \sqrt{\phi_{C}(x_C)}$, where $\theta_C$ can be any real-valued function, and where $v$ indicates the TN node dual to the clique $C$.}
\end{corollary}

Although standard BMs and UGMs operate very differently\textemdash and in the case of line graphs have been proven to have inequivalent expressive power~\cite{glasser2019}\textemdash we see that DBMs offer a unified means of representing both families of models with an identical parameterization. We further prove in the supplemental material that DBMs are equivalent in expressivity to \emph{locally purified states}, a model family generalizing both BMs and UGMs. 

Another motivation for the use of DBMs is in enabling conditional independence guarantees within the setting of quantum-inspired TN models. The ability to replace any decoherence operator by a fifth-order copy tensor with a visible edge lets us assign RVs to all decohered edges, such that marginalizing over these new RVs yields the original DBM distribution (Figure~\ref{fig:proofs}a). These new RVs behave identically to those of a UGM, letting us demonstrate a conditional independence property.

\begin{theorem}
\label{thm:cond_ind_dbm}
\Copy{cond_ind_dbm}{Consider a DBM with underlying graph $G$ and decohered edges $\decoherededges$, along with a subset $\cutedges \subseteq \decoherededges$ which forms a cut set for $G$. If we denote by $Z_C$ the set of RVs associated to $\cutedges$, and denote by $X_A$ and $X_B$ the sets of RVs associated to the two partitions of $G$ induced by the cut set $\cutedges$, then the DBM distribution satisfies the conditional independence property $\condind{X_A}{X_B}{Z_C}$.}
\end{theorem}

While the complete proof of Theorem~\ref{thm:cond_ind_dbm} is given in the supplemental material, the idea is simple (Figure~\ref{fig:proofs}d). The insertion of decoherence operators, which are examples of copy tensors, into the composite TN giving the DBM distribution allows any basis vector $e_z$ used for conditioning to be copied to all edges incident to the copy tensor. This in turn removes any direct correlations between the nodes on either side of the decohered edge, so that conditioning on a collection of RVs associated with a cut set of decohered edges results in a factorization of the conditional composite TN into a tensor product of two independent pieces.

\section{Experiments}

\begin{figure}
\centering
\includegraphics[width=\textwidth]{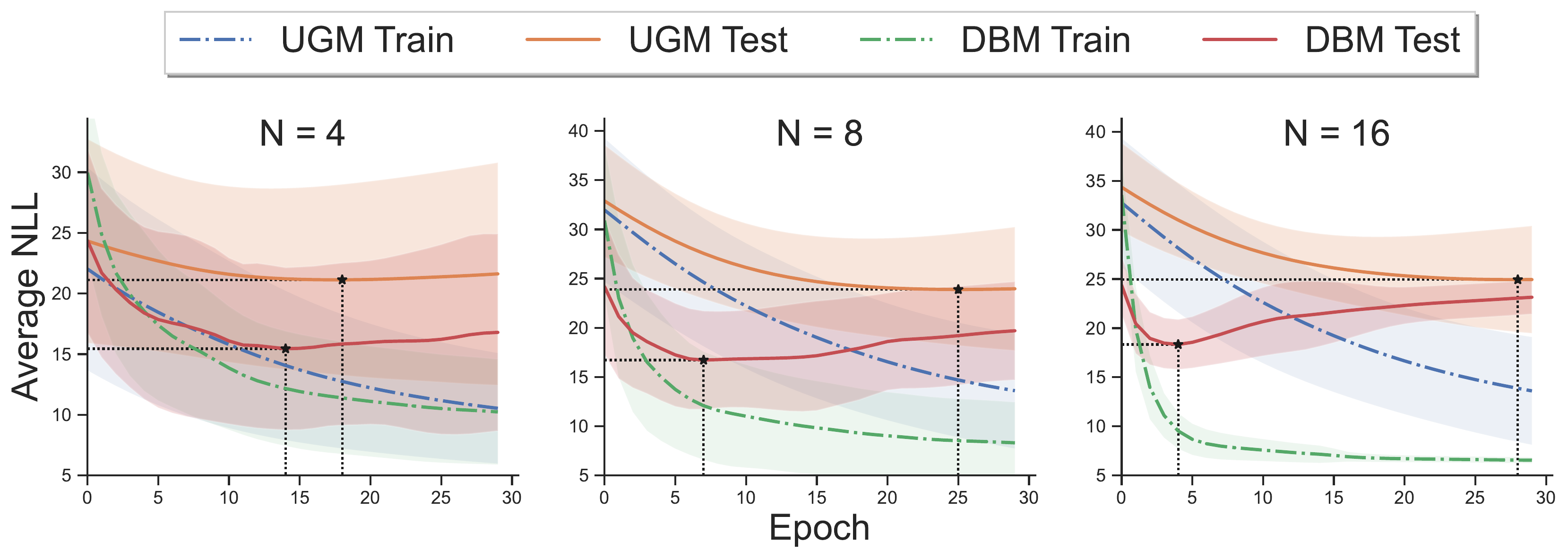}
\caption{DBM (red) vs. UGM (orange) models defined on a graph corresponding to a hidden Markov model (HMM). DBM exhibits a lower average negative log likelihood (NLL) on held-out data compared to an equivalent UGM on the Bars and Stripes dataset. Curves are means taken over 15 replications, shaded regions are 1 std dev. Star markers indicate minima. Left to right: increasing the number of hidden states $N$ in HMM leads to faster learning by the DBM relative to the UGM.}
\label{fig:hybrid_vs_hmm}
\end{figure}

Having discussed the \emph{representational} advantages of DBMs for structured discrete probability distributions, we now investigate the empirical advantages over UGMs with a similar graphical structure. We define both DBM and UGM models on an undirected version of the graph defining a hidden Markov model (HMM), and evaluate their relative performance on a sequential (flattened) version of the Bars and Stripes dataset~\cite{mackay2003information}. Bars and Stripes consists of two kinds of binary images: ones with only horizontal bars, and ones with only vertical stripes. Intuitively, we expect that the interference behavior available in DBMs can capture correlations more efficiently when the distribution being modeled exhibits regular periodic behavior. The 1D periodic patterns seen in the vertical stripe images give us an ideal setting for testing this out. 

Figure~\ref{fig:hybrid_vs_hmm} shows the results of optimizing a UGM and DBM to maximize marginal likelihood over observations of flattened $8\times 8$ images. Our models were implemented and trained with JAX~\cite{jax2018github}. The results are favorable to the DBM, which achieves both a lower average NLL on held-out data than the UGM, as well as a progressively smaller training times as the hidden dimensions $N$ of the models are increased. We conjecture that DBM models will more broadly have advantages in performance in modelling real-world data processes that exhibit long-term periodic behavior, such as geophysical processes, mechanical systems like human gait, and more general stochastic processes with some cyclical dynamics.

\section{Conclusion}
We use the physically-motivated notion of decoherence to define decohered Born machines (DBMs), a new family of probabilistic models that serve as a bridge between PGMs and TNs. As shown in Theorem~\ref{thm:fdbm_as_ugm} and Corollary~\ref{thm:ugm_as_dbm}, fully decohering a BM gives rise to a UGM, and conversely any subgraph of a UGM can be viewed as the decohered version of some BM. Crucial to this back-and-forth passage is the use of copy tensors, which further allows conditional independence guarantees in the context of TN modeling and provides an additional correspondence between the two modeling frameworks. An immediate limitation of our results surrounding DBMs is the focus on UGMs only. An extension to directed graphical models is left for future work, as is a deeper investigation into what kinds of problems could most benefit from utilizing quantum interference effects in the manner proposed. It is possible that DBMs would improve the performance of existing graphical model inference and learning algorithms by replacing sub-regions of the model with quantum-style ingredients, although a more systematic exploration of this question is needed. The integration of ``classical'' and ``quantum'' ingredients represented by a DBM further makes it a natural candidate for quantum machine learning, as decoherence represents a natural form of noise present in quantum hardware in the noisy intermediate-scale quantum (NISQ) era \cite{Preskill2018quantumcomputing}.

\begin{ack}
The authors thank Guillaume Verdon, Antonio Martinez, and Stefan Leichenauer for helpful discussions, and Jae Hyeon Yoo for engineering support. Geoffrey Roeder is supported in part by the National Sciences and Engineering Research Council of Canada (grant no. PGSD3-518716-2018).
\end{ack}

\bibliographystyle{plain}
\bibliography{references}

\newpage
\appendix

\section{Decohered Born Machines Compute Unnormalized Probability Distributions}

Here we show that every decohered Born machine (DBM) defines a valid (unnormalized) probability distribution, that is, that the tensor elements of a DBM are non-negative. This can be seen from the fact that the probability distribution represented by a DBM is obtainable as a marginalization of the distribution represented by a larger Born machine (BM) model. By definition, a DBM is a BM $\psi$ with the property that a subset $E_D$ of the set $E_H$ of the hidden edges of $\psi$ are decohered. Recall that decoherence here involves the contraction of $k$ third-order copy tensors $\Delta_3$, where $k$ is the size of the set $E_D$, and observe that such contraction can be achieved by marginalizing over new latent variables $Z_1,\ldots,Z_k$ introduced in the corresponding hidden edges. The tensor elements of the DBM will then take the form $\sum_{z_1,\ldots,z_k} |\psi'_{x_1,\ldots,x_n,\ldots,z_1,\ldots,z_k}|^2$, where $\psi'$ is associated to a larger BM containing all copy tensors $\Delta_3$ associated with decohered edges. These tensor elements are clearly non-negative, proving that DBMs always describe non-negative tensors.

\section{Expressivity of Decohered Born Machines}

\begin{figure}
\centering
\includegraphics[width=\textwidth]{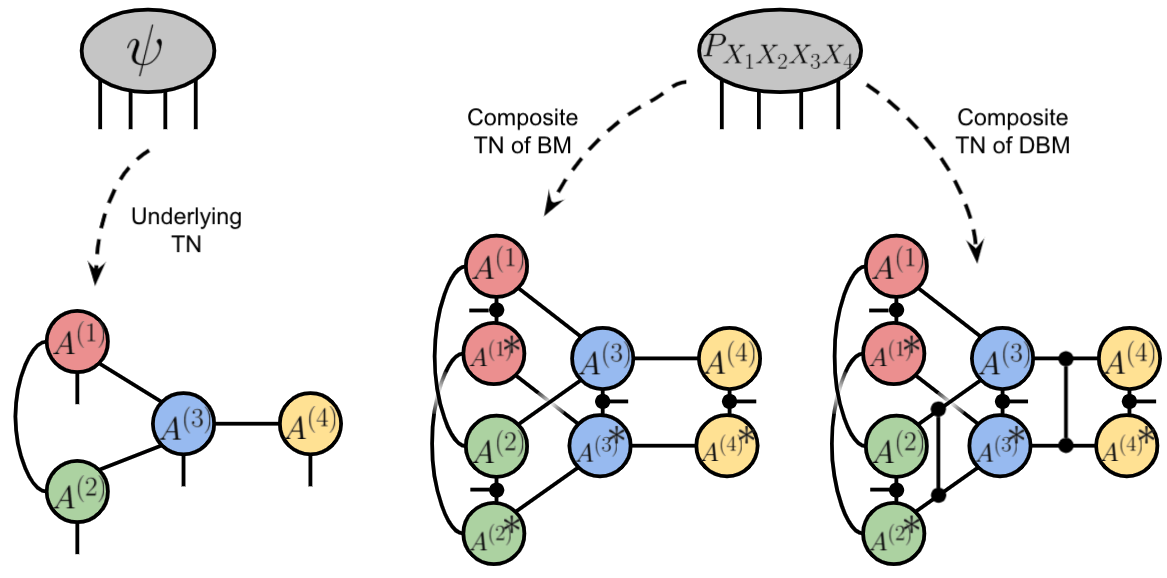}
\caption{A tensor network (TN) associated with a graph $G = (V, E)$ with $V = \{ 1, 2, 3, 4 \}$ and $E = \{ \edgevar_{1,2}, \edgevar_{2,3}, \edgevar_{1,3}, \edgevar_{3,4}, \edgevar_{1}, \edgevar_{2}, \edgevar_{3}, \edgevar_{4} \}$. The edge set $E$ is partitioned into visible edges $\visibleedges = \{ \edgevar_{1}, \edgevar_{2}, \edgevar_{3}, \edgevar_{4} \}$ and hidden edges $\hiddenedges = \{ \edgevar_{1,2}, \edgevar_{2,3}, \edgevar_{1,3}, \edgevar_{3,4} \}$, with the four visible edges giving a fourth-order tensor $\psi$. A Born machine (BM) uses two copies of this underlying TN, one with all cores complex-conjugated, to define a fourth-order composite TN whose associated tensor $P_{X_1 X_2 X_3 X_4}$ is an unnormalized probability distribution associated with four random variables. A decohered Born machine (DBM) uses a similar composite TN, but allows for a decoherence operator to be inserted in some edges, as specified by a set $\decoherededges \subseteq \hiddenedges$. In the composite TN shown, $\decoherededges = \{ \edgevar_{2,3}, \edgevar_{3,4} \}$.}
\label{fig:figure6}
\end{figure}

We prove several results which give a general characterization of the expressivity of decohered Born machines (DBMs), showing them to be capable of reproducing a range of classical and quantum-inspired probabilistic models. In Section~\ref{subsec:fdbm_as_ugm}, we prove that fully-decohered Born machines are equivalent in expressivity to undirected graphical models (UGMs), with the equivalence in question preserving the number of parameters of the two model classes. This result, which parallels the tautological equivalence of non-decohered DBMs and standard BMs, is followed in Section~\ref{subsec:dbm_as_lps} by a result showing the equivalence of DBMs and locally purified states (LPS), an expressive model class introduced in~\cite{glasser2019}.

We first review the definition of a DBM and some terminology for its graphical structure. Tensor networks (TNs) are defined in terms of a graph $G = (V, E)$ whose edges are allowed to be incident to either two or one nodes in $V$, and we will refer to the respective disjoint sets of edges as \emph{hidden edges} $\hiddenedges \subseteq E$ and \emph{visible edges} $\visibleedges \subseteq E$. Visible edges are associated with the modes of the tensor described by the TN, with the number of visible edges in $G$ equal to the order of the tensor. Nodes which aren't incident to any visible edges are referred to as \emph{hidden nodes} of the TN. Every node $v \in V$ is associated with a tensor \emph{core} $\TNnode$ of the TN, with the order of $\TNnode$ being equal to the degree of $v$ within $G$.

Recall that every BM is completely determined by a TN description of a higher-order tensor $\psi$, whose values are then converted into probabilities via the Born rule. We call the TN describing $\psi$ the \emph{underlying TN}, and the Born rule implies that the probability distribution can be described as a single \emph{composite TN} formed from two copies of the underlying TN, with all pairs of visible edges joined by copy tensors $\Delta_3$. We sometimes use the phrase \emph{composite edge} to refer to any pair of ``doubled up'' edges in the composite TN, in which case the composite TN can be seen as occupying the same graph as the underlying TN, but where each hidden edge corresponds to a composite edge. The probability distribution defined by a DBM is given by replacing certain composite edges in the composite TN by decoherence operators $\decoherence$, according to whether those edges belong to a set of \emph{decohered edges} $\decoherededges$ (Figure~\ref{fig:figure6}).

\subsection{Proof of Theorem~\ref{thm:fdbm_as_ugm}}
\label{subsec:fdbm_as_ugm}

We first restate Theorem~\ref{thm:fdbm_as_ugm}, before providing a complete proof.

\newtheorem*{thm:repeat_one}{Theorem~\ref{thm:fdbm_as_ugm}}
\begin{thm:repeat_one}
\Paste{fdbm_as_ugm}
\end{thm:repeat_one}

\begin{figure}
\centering
\includegraphics[width=\textwidth]{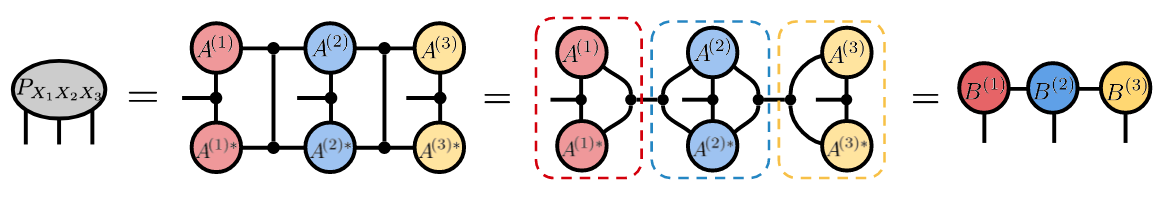}
\caption{Example of Theorem~\ref{thm:fdbm_as_ugm} showing the conversion of a fully-decohered BM into an equivalent UGM. By rewriting each decoherence operator as a product of third-order copy tensors, we can rewrite every pair of BM core tensors $\TNnode$ and $\TNnodestar$ as a single core tensor $\TNBnode$, whose values are guaranteed to be non-negative. This can consequently be used as a clique potential for a UGM.}
\label{fig:figure7}
\end{figure}

\begin{proof}

We show that the composite TN defining the probability distribution of a fully-decohered BM can be rewritten as a TN on the same graph $G$ as the underlying TN, with identical bond dimensions but where all cores take non-negative values. By virtue of the equivalence of non-negative TNs and UGMs~\cite[Theorem 2.1]{duality2018}, this suffices to prove Theorem~\ref{thm:fdbm_as_ugm}.

Fully-decohered BMs are defined as DBMs for which $\decoherededges = \hiddenedges$, so that every composite edge within the composite TN has been replaced by a decoherence operator $\decoherence$. Since $\decoherence = \Delta_4$, we can use the equality of different connected networks of copy tensors (Figure~\ref{fig:TN_diagrams}g) to express $\decoherence$ as a contraction of two third-order copy tensors $\Delta_3$ along a single edge. Decohered edges are hidden edges and are therefore incident to two distinct (pairs of) nodes in the composite TN. This allows us to move each copy of $\Delta_3$ onto a separate pair of nodes incident to the composite edge (Figure~\ref{fig:figure7}). We group together each pair of nodes $\TNnode$ and $\TNnodestar$, along with all copy tensors $\Delta_3$ incident to it, and contract each of these groups into a single tensor, which we call $\TNBnode$.

It is clear that each tensor $\TNBnode$ consists of a pair of cores $\TNnode$ and $\TNnodestar$ with all pairs of edges joined together by separate copies of $\Delta_3$. Since this arrangement of copy tensors corresponds to the element-wise product of $\TNnode$ and $\TNnodestar$, this implies that the elements of $\TNBnode$ satisfy $\TNBnode_{x_{\inc(v)}} = \left| \TNnode_{x_{\inc(v)}} \right|^2$, with $x_{\inc(v)}$ denoting the collection of indices associated with the edges incident to $v$ (these correspond to $x_C$ for some clique $C$ in the dual graph). Since each core $\TNBnode$ has the same shape as $\TNnode$, has non-negative values, and is arranged in a TN with the same graph as the underlying TN, this proves Theorem~\ref{thm:fdbm_as_ugm}.

\end{proof}

\subsection{Proof of Corollary~\ref{thm:ugm_as_dbm}}

\newtheorem*{thm:repeat_two}{Corollary~\ref{thm:ugm_as_dbm}}
\begin{thm:repeat_two}
\Paste{ugm_as_dbm}
\end{thm:repeat_two}

\begin{proof}

The statement of Corollary~\ref{thm:ugm_as_dbm} gives an explicit formula for constructing BM cores $\TNnode$ from clique potentials $\phi_C$, using an arbitrary real-valued tensor $\theta_C$. It can be immediately verified that the conversion from BM cores $\TNnode$ to effective clique potentials under full decoherence (Theorem~\ref{thm:fdbm_as_ugm}) recovers the same clique potentials we had started with, proving Corollary~\ref{thm:ugm_as_dbm}. Note that the values of the complex phases $\exp\left( 2\pi i\, \theta_C(x_C) \right)$ have no impact on the decohered cores.

\end{proof}

\subsection{Decohered Born Machines are Equivalent to Locally Purified States}
\label{subsec:dbm_as_lps}

\begin{figure}
\centering
\includegraphics[width=\textwidth]{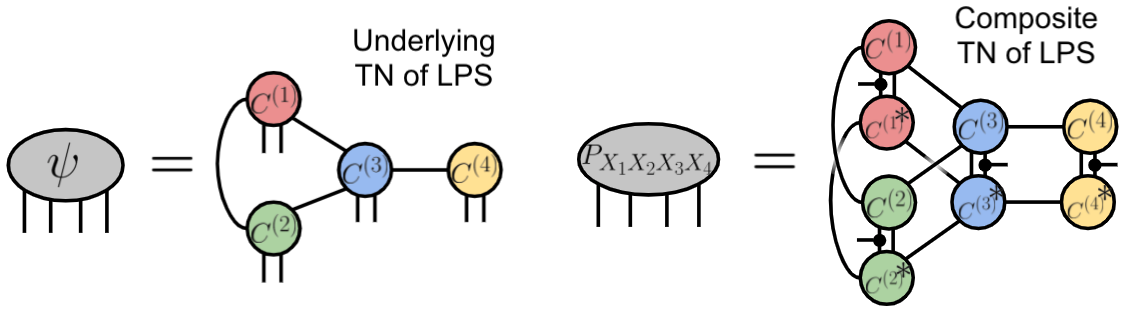}
\caption{A locally purified state (LPS) model is similar to a BM, but with an additional purification edge added to each node of the underlying TN. Although a small graphical change, this gives LPS greater expressive capabilities than BMs~\cite{glasser2019}. We show here that DBMs are equivalent to LPS in expressivity.}
\label{fig:figure8}
\end{figure}

Although the definition of locally purified states (LPS) in~\cite{glasser2019} assumes a one-dimensional line graph for the TN, we give here a natural generalization to LPS defined on more general graphs.

\begin{definition}
A locally purified state (LPS) consists of a tensor network over a graph $G$ containing $2n$ visible edges, where all cores contain exactly two visible edges, one of which is designated as a \emph{purification edge}, and the set of purification edges is denoted by $E_P \subseteq \visibleedges$. The $n$-variable probability distribution defined by an LPS is given by constructing the composite TN for a BM from these cores, with order $2n$, then marginalizing over all $n$ purification edges.
\end{definition}

An illustration of this model family is given in Figure~\ref{fig:figure8}. By choosing all purification edges to have dimension 1, LPS reproduce standard BMs, whereas~\cite[Lemma 3]{glasser2019} gives a construction allowing LPS to reproduce probability distributions defined by general UGM. Owing to this expressiveness, and to corresponding results for uniform variants of LPS~\cite{adhikary2021quantum}, we can think of LPS as representing the most general family of quantum-inspired probabilistic models. We now prove that DBMs are equivalent in expressivity to LPS, by first showing that LPS can be expressed as DBMs (Theorem~\ref{thm:lps_as_dbm}), and then showing that DBMs can be expressed as LPS (Theorem~\ref{thm:dbm_as_lps}).

\begin{theorem}
\label{thm:lps_as_dbm}
Consider an LPS whose underlying TN uses a graph $G = (V, E)$ with $n$ nodes, $2n$ visible edges, and $m$ hidden edges. The probability distribution represented by this LPS can be reproduced by a DBM over a graph with $2n$ nodes, $n$ visible edges, and $m+n$ hidden edges, where the decohered edge set $\decoherededges$ is in one-to-one correspondence with the purification edges $E_P$ of the LPS.
\end{theorem}

\begin{proof}

Starting with a given LPS, we construct a TN matching the description in the Theorem statement, whose interpretation as a DBM will recover the desired distribution. We begin with the underlying TN for the LPS, whose $n$ nodes each have one purification edge. We connect each purification edge to a new hidden node, whose associated tensor is the first-order copy tensor $\Delta_1$ with dimension equal to that of the purification edge. This converts all of the purification edges into hidden edges, which form the decohered edges of the DBM.

Given this new TN and choice of decoherence edges, the equivalence of the DBM distribution and the original LPS distribution arises from inserting decoherence operators $\decoherence$ in the composite edges connected to the new hidden nodes, and then using copy tensor rewriting rules to express the composite TN of the DBM as that of the LPS (Figure~\ref{fig:figure9}a). Given that the new hidden nodes are associated with constant tensors with no free parameters, and given that all of the cores defining the LPS are kept unchanged in the DBM, the overall parameter count is unchanged. This completes the proof of Theorem~\ref{thm:lps_as_dbm}.

\end{proof}

\begin{figure}
\centering
\includegraphics[width=0.7\textwidth]{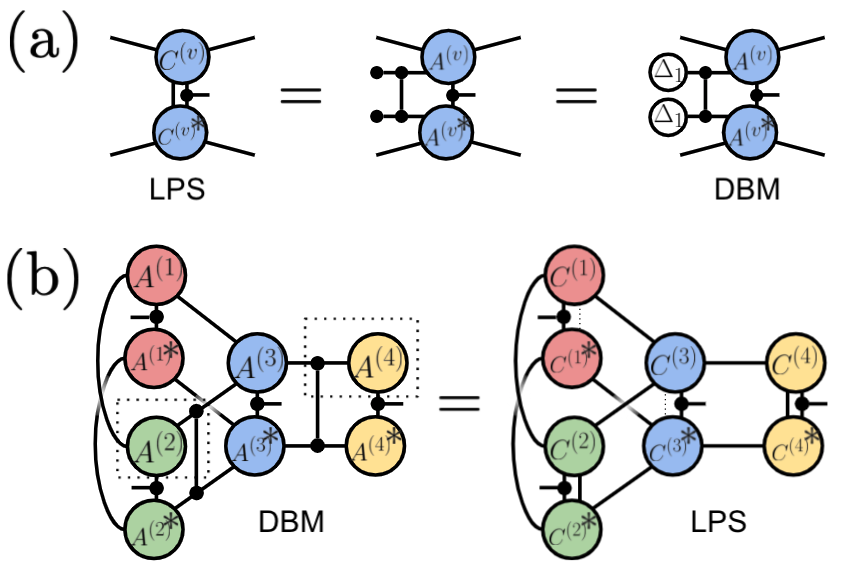}
\caption{(a) Conversion from an LPS to a DBM. Using diagram rewriting rules, each purification edge joining a pair of LPS cores is expressed as a larger network of copy tensors, which allows the edge to be seen as a decoherence operator $\decoherence$ between the original pair of nodes and a new pair of dummy nodes $\Delta_1$. The result is a DBM associated with an underlying TN with twice as many nodes, and one decohered edge for every purification edge in the LPS (b) Conversion from a DBM to an LPS. In this case, we choose a function $f$ mapping the decohered edges $\edgevar_{2, 3}$ and $\edgevar_{3, 4}$ to nodes 2 and 4, respectively. The dotted boxes show how this can be viewed as defining new cores $C^{(2)}$ and $C^{(4)}$ as the contraction of the DBM cores $A^{(2)}$ and $A^{(4)}$ with adjacent copy tensors $\Delta_3$. The result is an LPS, where we have used dotted edges to indicate trivial purification edges of dimension 1.}
\label{fig:figure9}
\end{figure}

\begin{theorem}
\label{thm:dbm_as_lps}
Consider a DBM defined on a graph $G = (V, E)$ with $n$ nodes and a set of decohered edges $\decoherededges \subseteq \hiddenedges$. Given any function $f: \decoherededges \to V$ assigning decohered edges to nodes of $G$ incident to those edges, we can construct an LPS with $n$ nodes which represents the same probability distribution as the DBM. This LPS is defined by a TN with an identical graphical structure to the TN underlying the original DBM, but with the addition of a purification edge at each node $v$ of dimension $d^{(v)} = \prod_{\edgevar \in f^{-1}(v)} d_\edgevar$, where $d_\edgevar$ is the bond dimension of edge $\edgevar$ and $f^{-1}(v)$ is the set of decohered edges mapped to node $v$.
\end{theorem}

\begin{proof}

Despite the somewhat complicated formulation of Theorem~\ref{thm:dbm_as_lps}, the idea is simple. In contrast to standard BMs, DBMs and LPS both permit direct vertical edges within the composite TN defining the model's probability distribution, and the proof consists of shifting these vertical edges from decohered edges to the nodes themselves. In the case where multiple vertical edges are moved to a single node, all of these can be merged into one single purification edge by taking the tensor product of the associated vector spaces. This gives the purification dimension $d^{(v)}$ appearing in the Theorem statement, with the overall procedure illustrated in Figure~\ref{fig:figure9}b. For nodes which are not assigned any decohered edges, a trivial purification edge of dimension $d^{(v)}=1$ is used. This completes the proof of Theorem~\ref{thm:dbm_as_lps}.

\end{proof}

\section{Proof of Theorem~\ref{thm:cond_ind_dbm}}

\newtheorem*{thm:repeat_three}{Theorem~\ref{thm:cond_ind_dbm}}
\begin{thm:repeat_three}
\Paste{cond_ind_dbm}
\end{thm:repeat_three}

\begin{figure}
\centering
\includegraphics[width=\textwidth]{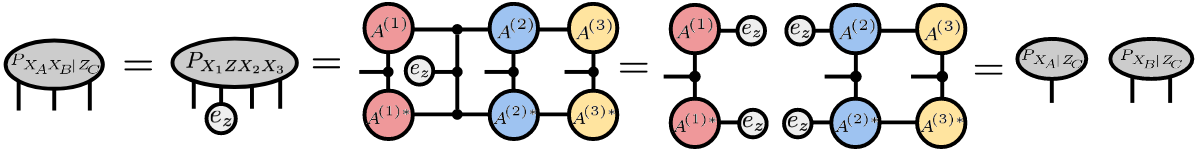}
\caption{Illustration of the conditional independence of decohered Born machines (DBMs), for a DBM over RV $X_1$, $X_2$, $X_3$, and $Z$. $Z_C := Z$ is a latent RV associated with the single decohered edge of the DBM, which is a cut set for the underlying graph. Conditioning on the value of $z$ of $Z$ splits the composite TN into two independent pieces, with the result being a probability distribution where $X_A := X_1$ and $X_B := (X_2, X_3)$ are independent RVs.}
\label{fig:figure10}
\end{figure}

\begin{proof}

From the definition of a cut set, the removal of $\cutedges$ from the graph for the underlying TN partitions $G$ into two disjoint pieces, and the same property holds true for the composite TN giving the DBM probability distribution. Figure~\ref{fig:figure10} illustrates how conditioning on a decohered edge of a DBM results in the splitting of the associated decoherence operator into a tensor product of two rank-1 matrices, which propagate the value of the conditioning value $z$ to both pairs of incident nodes. Consequently, each composite edge whose value is conditioned on will be removed from the composite TN, and if the set of conditioning edges form a cut set for $G$, this will result in the separation of the post-conditional composite TN into two disconnected pieces. This implies the independence of the composite random variables $X_A$ and $X_B$ in the conditional distribution, which completes our proof.

\end{proof}

\section{Experimental Details}
The code used to generate Figure~\ref{fig:hybrid_vs_hmm} is given in the included notebook, which will install necessary libraries, retrain the models, and regenerate the figure if run in sequence. Instructions are in a README file. The saved parameters are also included. We use NumPy \cite{harris2020} (BSD-compatible licence) and JAX \cite{jax2018github} (Apache 2.0 licence) for scientific computation, and Matplotlib \cite{Hunter:2007} (GPL-compatible licence) for visualization. In this section, we discuss the details needed for independent implementation and verification.

Figure~\ref{fig:hybrid_vs_hmm} answers the model selection problem: which model family yields better performance on held-out data? We implement a Hidden Markov Model (HMM) as both a UGM and as a DBM, and learn parameters to maximize marginal likelihood on the Bars and Stripes dataset~\cite{mackay2003information}.

\subsection{Model Complexity: Ensuring Equal UGM and DBM Parameter Counts}
To make a fair comparison between models, we must ensure that they have identical model complexity, measured here as total parameter count. On the same underlying graph, a DBM will have double the number of parameters as a UGM, because each parameter in a DBM is complex-valued and has both a real and an imaginary component.
To match parameter counts, we implement a mixture of two UGMs with independent parameters $\Phi_1$ and $\Phi_2$, where each $\Phi_i$ consists of a collection of clique potential values. The probability of an observation ${\bf O} = \{O_1, O_2, \dots, O_M\}$, $P_{UGM}(O_1, O_2, \dots, O_M)$, is given by the convex combination
\begin{align}
P_{UGM}({\bf O}) = \lambda_{UGM} P({\bf O}; \Phi_1) + (1 - \lambda_{UGM}) P({\bf O}; \Phi_2),
\end{align}
where $P({\bf O}; \Phi)$ denotes the probability assigned by a UGM with clique potentials functions contained in $\Phi$. Note that the UGM is now represented as a mixture distribution where the mixture weight $\lambda_{UGM}$ is found as a point estimate. To ensure the models are representing the same family of distributions, we also make the DBM a mixture. We do so by splitting the weights of a DBM into their magnitude and phase components, with the former specified by clique potentials in the parameter set $\Phi$ and the latter by phase functions in a disjoint parameter set $\Theta$. The magnitude components are used to compute an HMM as in the UGM case, yielding a component $P({\bf O}; \Phi_3)$. The complex phase components $\Theta_3$ are then used to assign complex phases to the magnitudes in $\Phi_3$, leading to a Born machine probability distribution denoted by $P({\bf O}; \Phi_3, \Theta_3).$ The probability of an observation ${\bf O}$ is given by
\begin{align}
\log P_{DBM}({\bf O}) = \log \left( \lambda_{DBM} P({\bf O}; \Phi_3) + (1 - \lambda_{DBM}) P({\bf O};  \Phi_3, \Theta_3) \right),
\end{align}
where we again make a point estimate of the mixture weight $\lambda_{DBM}$.
By sharing $\Phi_3$ between the two components of the DBM, we match parameter counts between the two models, with both models containing $2 |\Phi| + 1$ parameters in total.

\subsection{Dataset, Training, and Hyperparameters}

For the experiments, we generated a Bars and Stripes dataset of $8 \times 8$ images, and varied the HMM hidden dimension $N$ between $2,4,8,16,32$. Code for generating the dataset is included in the UGM and DBM training notebook, as well as a standalone notebook. To represent the images as 1D sequences, we use a horizontal raster scan from upper left to bottom right. Figure~\ref{fig:hybrid_vs_hmm} reports results for $N=\{4,8,16\}$ with 2 and 32 having given similar results to 4 and 16. The observation sequence length was fixed to $T=16$.

Training for each $N$ was replicated 15 times and trained on a different 70\% train, 30\% test split. The sequence of splits was fixed among different $N$, controlled by initializing the random seed to 0 for each experiment.

Parameters for each model are: hidden-to-hidden transition probabilities $T$, observation probabilities $H$, and distribution over the initial hidden state $v$, and convex mixture $\lambda$. For the UGM, we have two independent sets of $T,H,v$. For the DBM, we have an additional set of complex phases $T_c, H_c, v_c$. Both probabilities and phases are given an unconstrained parameterization in log space. The mixture component was parameterized as the inverse of a sigmoid, e.g., an unconstrained real number. Probabilities and mixture weights were initialized as independent standard normal random variables. The complex phase components were independently initialized as $\omega_j = i 2 \pi q_j$, where $i = \sqrt{-1}$ and $q_j \sim \operatorname{Unif}[0,1]$.

We conducted full batch gradient descent on the training split for 30 epochs for each $N$ and each replication, using an adaptive gradient optimizer (\cite{kingma2014adam}). Experiments were run on a MacBookPro15,1 model with a single 2.6 GHz 6-Core Intel i7 and 16GB of memory. Each curve in Figure~\ref{fig:hybrid_vs_hmm} is the mean over 15 replications, and the shaded areas are 1 standard deviation.

\section{Gauge Freedom in Probabilistic Tensor Networks}

\begin{figure}
\centering
\includegraphics[width=\textwidth]{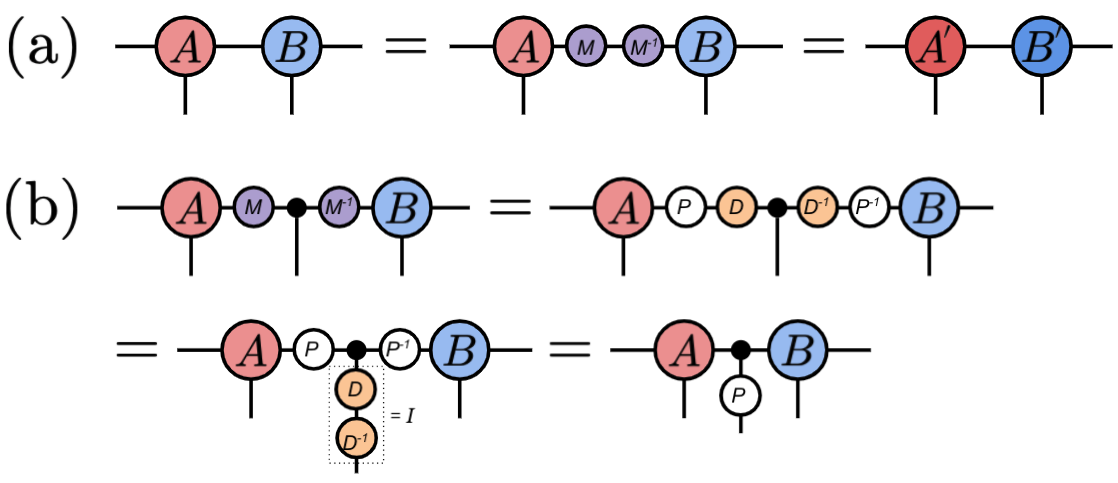}
\caption{(a) Action of a gauge transformation on a hidden edge of a TN. The insertion of an invertible matrix $M$ and its inverse leads to the adjacent tensor cores $A$ and $B$ being transformed into new tensor cores $A'$ and $B'$, which nonetheless describe the same overall tensor when all hidden edges are contracted together. The use of copy tensors in TNs generally restricts this gauge freedom. (b) The restriction of a TN to have cores with entirely non-negative entries forces any gauge transformation to factorize as $M = PD$, for $P$ a permutation and $D$ a diagonal matrix with strictly positive entries. We show how this restricted gauge freedom mostly commutes with any copy tensor inserted into the hidden edge, with copy tensor rewriting rules allowing us to express such gauge transformations as a trivial permutation of the outcomes of the latent RV associated with the hidden edge. This explains why hidden edges of non-negative TNs can be expressed as latent RVs without loss of generality, allowing a faithful representation as a UGM.}
\label{fig:figure11}
\end{figure}

Tensor networks matching the description given in Definition~\ref{def:tn} exhibit a form of symmetry in their parameters commonly referred to as \emph{gauge freedom}. This symmetry is generated by edge-dependent \emph{gauge transformations}, wherein an invertible matrix $M$ and its inverse $M^{-1}$ are inserted in a hidden bond of a TN, and then applied to the two tensor cores on nodes incident to that hidden edge. This results in a change in the parameters of the two incident core tensors, which nonetheless leaves the global tensor parameterized by the TN unchanged. The phenomenon of gauge freedom ultimately arises from the close connection between TNs and multilinear algebra, where gauge transformations on a given hidden edge correspond to changes in basis in the vector space associated with the hidden edge.


The use of copy tensors in a TN leads to a preferred choice of basis, and thereby breaks the full gauge freedom of any edge incident to a copy tensor node. It is therefore surprising that for non-negative TNs, i.e. those with all core tensors taking non-negative values, hidden edges were shown to be expressable as latent RVs without loss of generality, via the insertion of copy tensors in the hidden state space (Section~\ref{sec:ugm_as_nntn}). The generality of this operation means that any non-negative TN can be converted into a UGM by associating hidden edges with latent RVs, where the original distribution over only visible edges is recovered by marginalizing over hidden edges. This fact is a key ingredient in the exact duality between non-negative TNs and UGMs, and differs from quantum-style probabilistic TN models. For example, attempting to observe the latent states associated to hidden edges in a BM will generally lead to a change in the distribution over visible edges, even after marginalizing out these new latent RVs.

We observe here that the generality in associating hidden edges of a non-negative TN to latent RVs is a consequence of the fact that \emph{non-negative TNs already have significantly diminished gauge freedom}. More precisely, in order for a gauge transformation on a hidden edge to maintain the non-negativity of both tensor cores incident to that edge, we must generally have the change of basis matrix $M$, as well as its inverse $M^{-1}$, possess only non-negative entries. This is a strong limitation, and is equivalent to the gauge transformation factorizing as a product $M = PD$, where $P$ is a permutation matrix and $D$ is a diagonal matrix with strictly positive entries~\cite{demarr1972}. We illustrate in Figure~\ref{fig:figure11} how this restricted gauge freedom maintains the overall structure of the copy tensor inserted into a hidden edge, with the result being an irrelevant permutation of the discrete values of the hidden latent variable associated with that edge. 

The situation is quite different for BMs and DBMs, and we remark that the use of decoherence operators in a DBM means that the gauge freedom of such models is different than for BMs. In particular, two BMs whose underlying TNs are related by gauge transformations will necessarily define identical distributions, whereas the corresponding DBMs resulting from decohering some gauge-transformed hidden edges may define different distributions. In this sense, the appropriate notion of gauge freedom for a DBM lies in between that of a BM and a UGM defined on the same graph, in a manner set by the pattern of decohered edges.

The choice of basis in which decoherence is performed can be treated as an additional parameter of the model, and we view the interaction between this basis-dependence of decoherence and basis-fixing procedures related to TN \emph{canonical forms} an interesting subject for future investigation.





\end{document}